\documentclass[runningheads]{llncs}
\usepackage[T1]{fontenc}
\usepackage{graphicx}
\usepackage{booktabs}
\usepackage[misc]{ifsym}

\usepackage{amsmath}
\usepackage{amsfonts}
\usepackage{amssymb}
\usepackage{xfrac}
\usepackage{upgreek}
\usepackage{times}
\usepackage{subfig}

\usepackage[numbers,sort&compress,sectionbib]{natbib}

\DeclareMathOperator*{\E}{\mathbb{E}}

\DeclareMathOperator{\kl}{D_{KL}}

\begin{document}

\title{Learning Overspecified Gaussian Mixtures Exponentially Fast with the EM Algorithm}

\titlerunning{Learning Overspecified Gaussian Mixtures}

\author{Zhenisbek Assylbekov\inst{1,2}\and Alan Legg\inst{1}\and Artur Pak\inst{2,3}}
\authorrunning{Assylbekov et al.}

\institute{Department of Mathematical Sciences, Purdue University Fort Wayne, Fort Wayne IN, USA 
\\\email{\{zassylbe,leggar01\}@pfw.edu}
\and
Department of Mathematics, Nazarbayev University, Astana, Kazakhstan
\and
Mohamed bin Zayed University of Artificial Intelligence, Abu Dhabi, UAE
\\\email{artur.pak@nu.edu.kz}
}

\maketitle

\begin{abstract}
We investigate the convergence properties of the EM algorithm when applied to overspecified Gaussian mixture models---that is, when the number of components in the fitted model exceeds that of the true underlying distribution. Focusing on a structured configuration where the component means are positioned at the vertices of a regular simplex and the mixture weights satisfy a non-degeneracy condition, we demonstrate that the population EM algorithm converges exponentially fast in terms of the Kullback-Leibler (KL) distance. Our analysis leverages the strong convexity of the negative log-likelihood function in a neighborhood around the optimum and utilizes the Polyak-Łojasiewicz inequality to establish that an $\epsilon$-accurate approximation is achievable in $O(\log(1/\epsilon))$ iterations. Furthermore, we extend these results to a finite-sample setting by deriving explicit statistical convergence guarantees. Numerical experiments on synthetic datasets corroborate our theoretical findings, highlighting the dramatic acceleration in convergence compared to conventional sublinear rates. This work not only deepens the understanding of EM's behavior in overspecified settings but also offers practical insights into initialization strategies and model design for high-dimensional clustering and density estimation tasks.

\end{abstract}


\section{Introduction and Main Results}\label{sec:intro} 

Let \( {Z}_1, \ldots, {Z}_n \) be a random sample from the standard $d$-variate normal distribution \( \mathcal{N}_d({0}, {I}) \), where \( {0} \in \mathbb{R}^d \) is the mean vector, and \( {I} \in \mathbb{R}^{d \times d} \) is the identity covariance matrix. We aim to fit a \( k \)-component Gaussian mixture model of the form  
\begin{equation}
\pi_1 \cdot \mathcal{N}_d(\mu_1, I) + \ldots + \pi_k \cdot \mathcal{N}_d(\mu_k, {I}) \label{eq:gmm}    
\end{equation}
to this sample. When \( k \geq 2 \), this setting is known as \emph{overspecification}, meaning the fitted model contains more mixture components than the true data-generating process. We assume the location parameters \( \mu = (\mu_1^\top, \ldots, \mu_k^\top)^\top \) are unknown, while the mixture weights \( (\pi_1, \ldots, \pi_k) \) are fixed and satisfy \( \pi_j > 0 \) and \( \sum_{j=1}^k \pi_j = 1 \). 

Let \( f({x}; {\mu}) \) denote the probability density function of the mixture defined in \eqref{eq:gmm}. The maximum likelihood estimator (MLE) of \( {\mu} \) is given by  
\begin{equation}
\hat{\mu} \in \arg\max_{\mu} \frac{1}{n} \sum_{i=1}^n \log f({Z}_i; {\mu}). \label{eq:mle}
\end{equation}
For \( k \neq 1 \), a closed-form solution for \( \hat{\mu} \) does not exist. Instead, \eqref{eq:mle} is typically solved using iterative methods such as the Expectation-Maximization (EM) algorithm \cite{https://doi.org/10.1111/j.2517-6161.1977.tb01600.x}. However, since the log-likelihood function in \eqref{eq:mle} is non-concave, iterative methods generally do not guarantee convergence to the global optimum.

Recent studies have analyzed the behavior of EM in overspecified settings. \citet{DBLP:conf/aistats/DwivediHKWJ020,dwivedi2020singularity} examined the case \( k = 2 \), differentiating between balanced mixtures (\( \pi_1 = \pi_2 = 1/2 \)) and unbalanced mixtures (\( \pi_1 \neq \pi_2 \)). Assuming symmetric means (\( \mu_1 = -\mu_2 \)), they showed that in the unbalanced case, the population EM\footnote{Population EM assumes access to the true data-generating distribution, allowing updates to be computed as exact expectations, free from sampling variability.} algorithm requires \( O(\log(1/\epsilon)) \) steps to obtain an \( \epsilon \)-accurate estimate of the parameter \( \mu^\ast = 0 \). In contrast, for balanced mixtures, the algorithm needs \( \Theta(\log(1/\epsilon)/\epsilon^2) \) steps, making it exponentially slower.

\citet{xu2024toward} investigated the behavior of \emph{gradient EM}\footnote{Gradient EM replaces the M-step of the Expectation-Maximization algorithm with a single gradient ascent step on the Q-function.} in the population setting for general \( k \). Their results show that gradient EM exhibits a slow convergence rate, requiring \( O(1/\epsilon^2) \) iterations to approximate the $k$-component Gaussian mixture \eqref{eq:gmm} to \( \mathcal{N}_d({0}, {I}) \) within an accuracy \( \epsilon \) in the KL metric. Their work imposes no assumptions on the balance of the mixture weights or the arrangement of Gaussian component centers. From this perspective, their result is more general. However, as demonstrated by \citet{dwivedi2020singularity}, in certain overspecified cases, the EM algorithm can achieve exponential convergence. This motivates the following question:
\begin{center}\vspace{5pt}\emph{When learning a mixture of \( k \) Gaussians from $\mathcal{N}_d(0,I)$ data, does there exist a configuration of component centers and mixture weights such that the EM algorithm converges exponentially fast?}\vspace{5pt}\end{center}
Our answer to this question is affirmative, and we present it in the form of the following theorem.
\begin{theorem}\label{thm:population}
Let \( R\in\mathbb{R}^{d\times d} \) be an orthogonal matrix such that for any nonzero $\theta\in\mathbb{R}^d$, the points
\[
\mu_j(\theta) = R^{j-1} \theta, \quad \text{for } j=1,\ldots,k.
\] form the vertices of a regular \((k-1)\)-simplex in $\mathbb{R}^d$, $d\ge k-1$, centered at the origin. 
Consider the \( k \)-component Gaussian mixture
$$
\mathcal{G}(\theta) := \pi_1 \cdot \mathcal{N}_d(\mu_1(\theta), I) + \pi_2 \cdot \mathcal{N}_d(\mu_2(\theta), I) + \dots + \pi_k \cdot \mathcal{N}_d(\mu_k(\theta), I),
$$
where the mixture weights \( \pi_1, \dots, \pi_k \) are fixed, positive,  satisfy \( \sum_{j=1}^{k} \pi_j = 1 \), and their discrete Fourier transform has no zero entries. This mixture is fitted to the standard Gaussian distribution \( \mathcal{N}(0, I) \) using the Population EM algorithm. Let \( \theta_t \) denote the parameter value at iteration \( t \). Then there exists $\gamma>0$ such that the following holds:
\[
D_{\text{KL}}[\mathcal{N}(0,I) \parallel \mathcal{G}(\theta_t)] \leq \kappa^t D_{\text{KL}}[\mathcal{N}(0,I) \parallel \mathcal{G}(\theta_0)],
\]
for $\theta_0$ satisfying $\|\theta_0\|\le\gamma$ and for some constant \( \kappa \in (0,1) \).
\end{theorem}

At first glance, the choice of placing Gaussian component centers at the vertices of a regular \((k-1)\)-simplex may seem arbitrary. However, this configuration naturally arises in the context of Gaussian mixture learning. 

A common approach to initializing Gaussian mixture components in the EM algorithm is via Lloyd's variant of the \( k \)-means algorithm \cite{DBLP:journals/tit/Lloyd82}. We can show that the vertices of a regular \((k-1)\)-simplex (with a particular radius) form a fixed point of Lloyd's algorithm when applied to \( \mathcal{N}(0, I) \) at the population level (Section~\ref{sec:kmeans}). Figure~\ref{fig:kmeans} illustrates this for finite samples and $k=2,3$.
\begin{figure}[htbp]
    \centering
    \subfloat[$k=3$ on $\mathcal{N}_2(0,I)$]{%
        \includegraphics[width=0.45\textwidth]{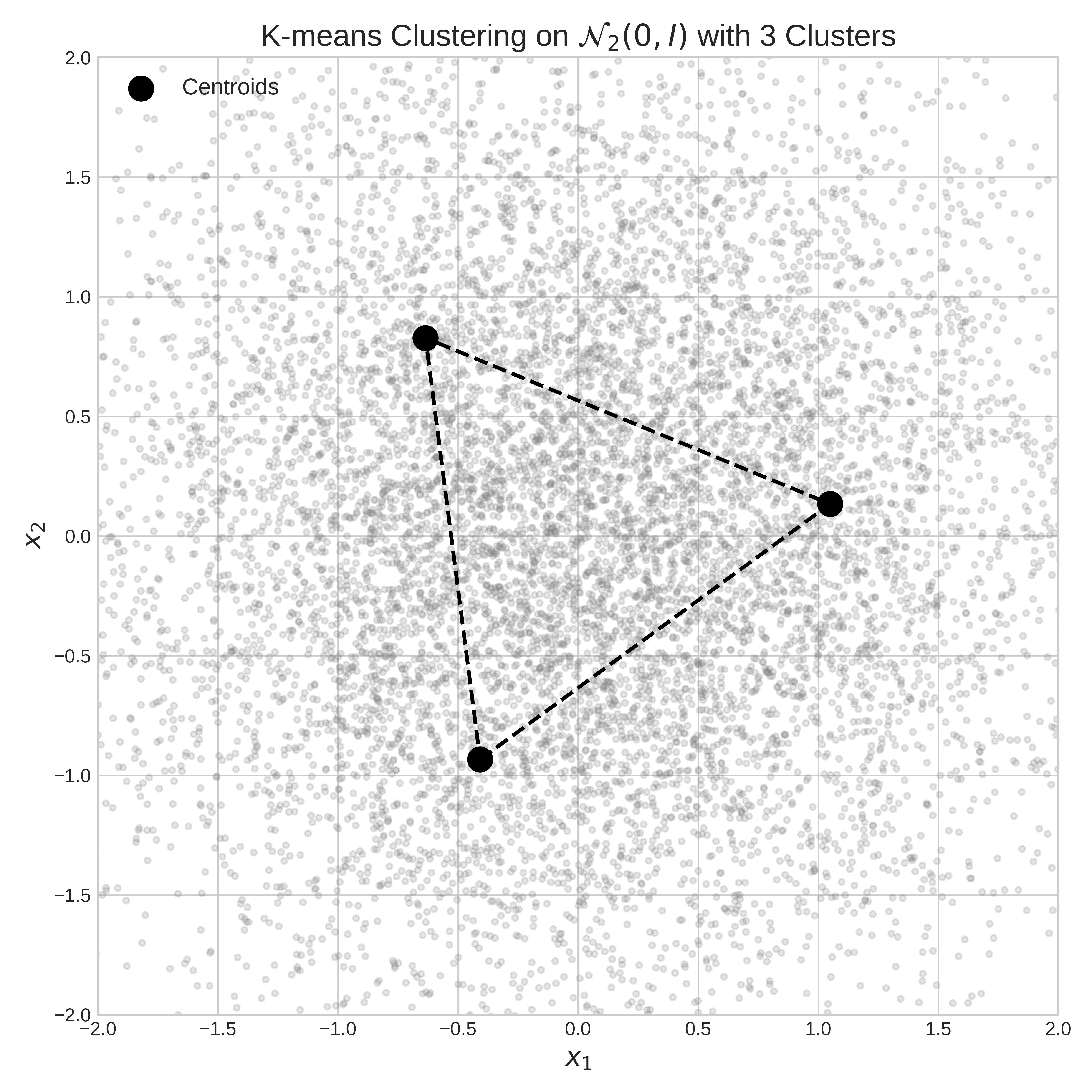}%
    }
    \hfill
    \subfloat[$k=4$ on $\mathcal{N}_3(0,I)$]{%
        \includegraphics[width=0.45\textwidth]{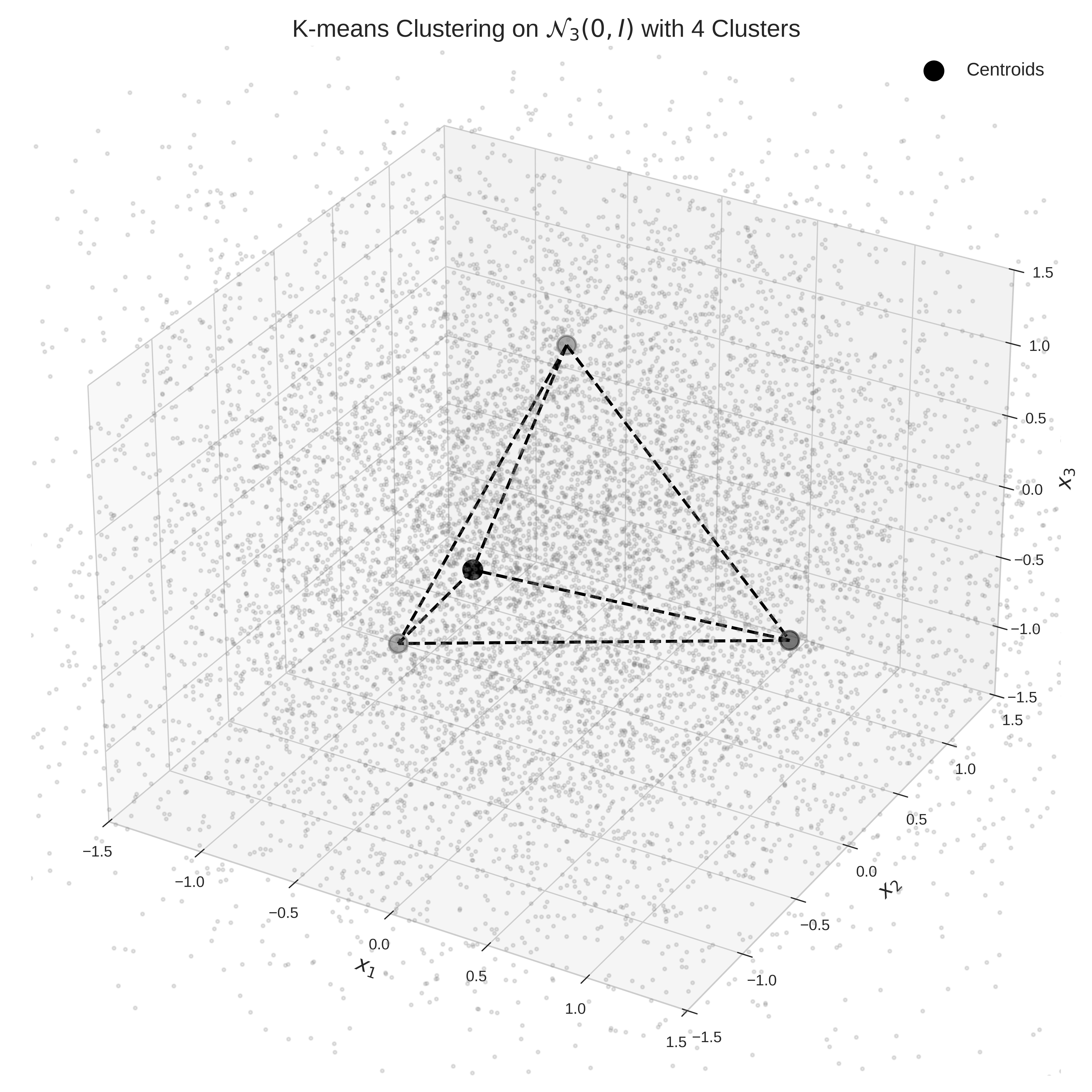}%
    }
    \caption{K-means clustering on standard Gaussian data. \emph{Left:} 10,000 samples in $\mathbb{R}^2$ are clustered into 3 groups; centroids (black markers) are connected by dashed lines to form a near-equilateral triangle. \emph{Right:} 10,000 samples in $\mathbb{R}^3$ are clustered into 4 groups; centroids (black markers) connected by dashed lines approximate a regular tetrahedron.}
    \label{fig:kmeans}
\end{figure}
This suggests that the regular \((k-1)\)-simplex is a natural initialization choice for the EM algorithm when learning an overspecified Gaussian mixture from data generated by a single Gaussian.

Following \citet{xu2024toward}, we focus on the convergence of the fitted distribution to the true distribution in the KL metric rather than the convergence of the parameters to zero in the Euclidean metric, as studied by \citet{dwivedi2020singularity}. However, our analysis fundamentally differs from both works. We find that the expected negative log-likelihood function is strongly convex in the neighborhood of the optimum and satisfies the so-called \emph{Polyak-Łojasiewicz} inequality \citep{POLYAK1963864,lojasiewicz1963topological}. This significantly simplifies the analysis of the convergence of the KL distance between the fitted model and the true distribution.

An immediate consequence of Theorem~\ref{thm:population} is that the Population EM algorithm requires $O\left(\log\left(1/\epsilon\right)\right)$ steps to approximate the mixture \( \mathcal{G}(\theta) \) to \( \mathcal{N}(0, I) \) within \( \epsilon \) in the KL metric. This is exponentially faster than the general result of \citet{xu2024toward}. Moreover, by leveraging the now-standard approach of \citet{DBLP:journals/corr/BalakrishnanWY14}, we can translate the fast convergence of the population EM  into the following finite-sample guarantee for the sample-based EM algorithm.

\begin{theorem}\label{thm:main}
    Under the assumptions of Theorem~\ref{thm:population} on the structure of the Gaussian mixture, there exists $\gamma>0$ such that for any initialization $\theta_0$ with $\|\theta_0\|\le\gamma$, the EM algorithm produces a sequence of parameter estimates \( \hat{\theta}_t \) satisfying
    \begin{equation}
        D_\text{KL}\left[\mathcal{N}({0},{I})\parallel\mathcal{G}(\hat\theta_T)\right]\le c_1\|\theta_0\|^2\frac{\log(1/\delta)}{n},\label{eq:kl_bound}
    \end{equation}
    for \( T \ge c_2 \log\frac{n}{\log(1/\delta)} \) with probability at least \( 1 - \delta \).
\end{theorem}

The proof of Theorem~\ref{thm:main} is based on a perturbation bound that relates the sample-based EM operator to its population-level counterpart (Lemma~\ref{lem:perturb} in Section~\ref{app:perturb_bound}). In turn, the proof of the latter utilizes standard arguments to derive Rademacher complexity bounds.

The theoretical insights presented above are  supported by our numerical experiments. In particular, Figure~\ref{fig:population_em} demonstrates the exponential decay of the KL divergence over EM iterations under various mixture weight configurations, while Figure~\ref{fig:sample_em} reveals how the divergence decreases as the sample size increases. Together, these figures provide an intuitive visualization of the convergence dynamics and statistical guarantees established by our analysis.

\begin{figure}[htbp]
    \centering
    \includegraphics[width=0.65\linewidth]{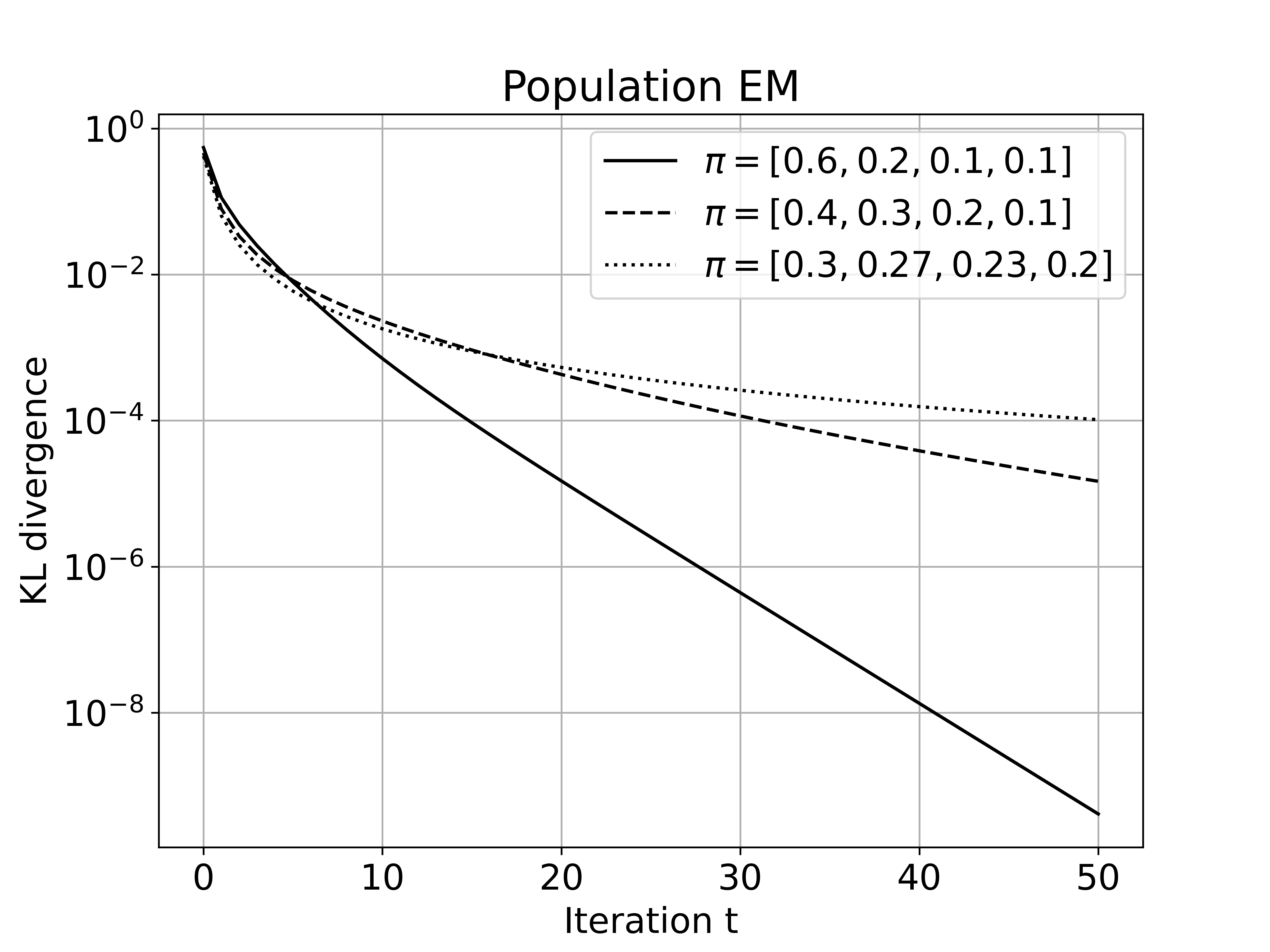}
    \caption{Convergence of Population EM: The plot shows the evolution of the KL divergence versus the number of EM iterations for three different sets of mixture weights. The curves correspond to varying levels of imbalance illustrating how the choice of weights influences convergence speed.}
    \label{fig:population_em}
\end{figure}
\begin{figure}[htbp]
    \centering
    \includegraphics[width=0.65\linewidth]{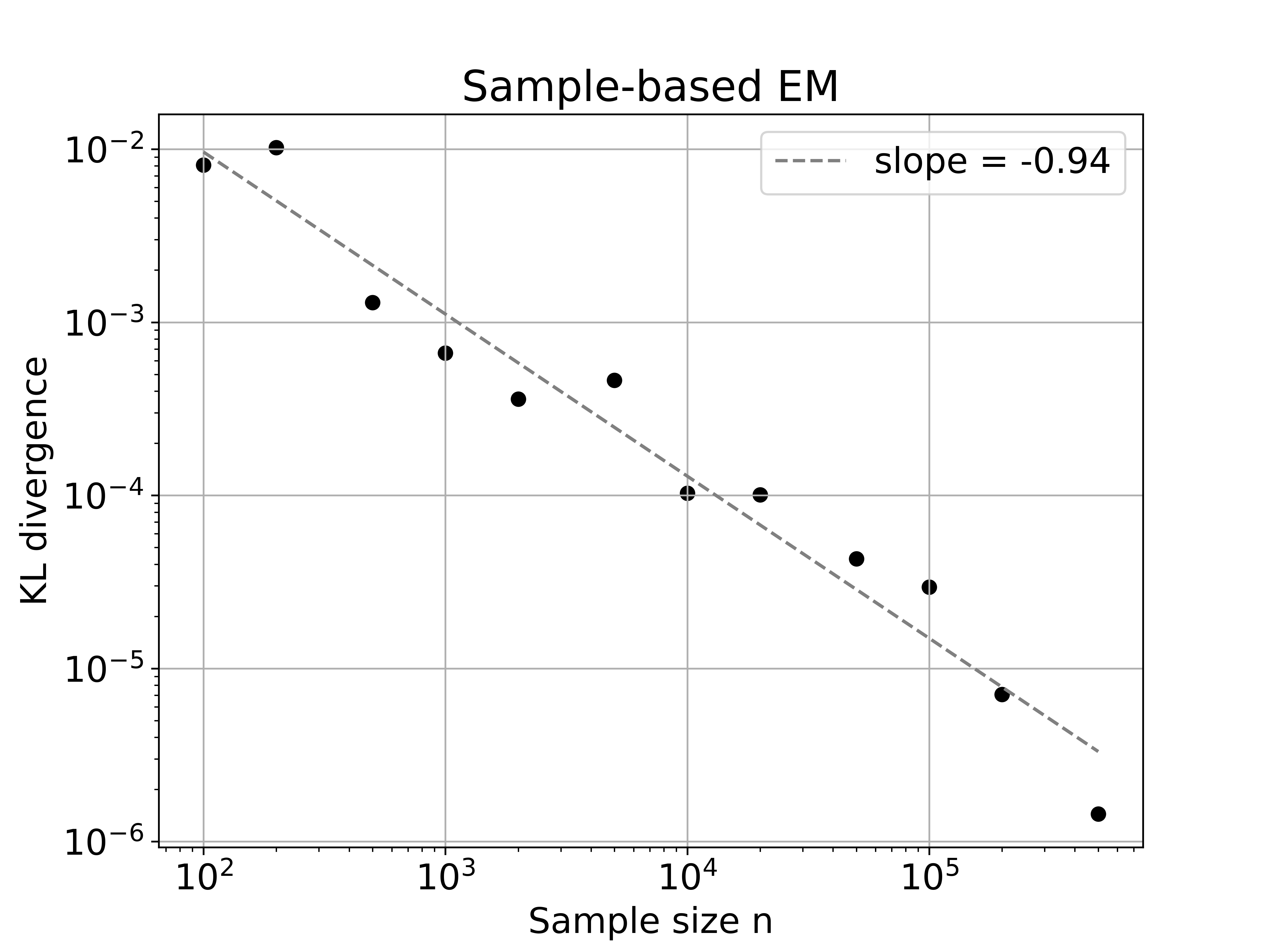}
    \caption{Sample-Based EM Performance: The figure plots the final KL divergence against the sample size $n$ on a log-log scale. It demonstrates how increasing the number of samples improves the accuracy of the EM estimate by reducing the divergence between the fitted mixture and the true 
    $\mathcal{N}(0,I)$ distribution}
    \label{fig:sample_em}
\end{figure}

To summarize, our work makes the following key contributions:

\begin{itemize}
    \item We demonstrate that the EM algorithm can achieve \emph{exponential convergence} in the KL metric when learning an overspecified mixture of \( k \) Gaussian components under a specific structured configuration of mixture centers and weights. This contrasts with prior work \citep{xu2024toward}, which establishes only sublinear convergence rates in general settings.

    \item We develop a novel analytical framework based on the \emph{Polyak-Łojasiewicz inequality}, leveraging the strong convexity of the expected negative log-likelihood function near the optimum. This significantly simplifies the convergence analysis compared to previous approaches.

    \item We establish an explicit \emph{finite-sample guarantee} for learning an overspecified mixture of \( k \) Gaussians with the EM algorithm.  
\end{itemize}

These contributions provide new insights into the role of mixture structure in the efficiency of EM and identify settings where the algorithm achieves fast convergence rates.

\paragraph{Notation.}
Lowercase letters (${x}$) denote vectors in $\mathbb{R}^d$, uppercase letters (${A}$, ${X}$) denote matrices and random vectors. The Euclidean norm is denoted by $\|{x}\|:=\sqrt{{x}^\top{x}}$. We denote $\{1,2,\ldots,k\}$ with $[k]$.

The probability density function of ${Z}\sim\mathcal{N}({0},{I})$, where ${I}$ is a $d\times d$ identity matrix, is denoted by $\phi({z})$. The cumulative distribution function of $Z\sim\mathcal{N}(0,1)$ is denoted by $\Phi(z)$. 

Given $f:\,\mathbb{R}\to\mathbb{R}$ and $g:\,\mathbb{R}\to\mathbb{R}_+$, we write $f\precsim g$  if there exist $x_0\in\mathbb{R}$, $c\in\mathbb{R}_+$ such that for all $x>x_0$ we have $|f(x)|\le c g(x)$. When $f:\mathbb{R}\to\mathbb{R}_+$, we write $f\asymp g$  if $f\precsim g$ and $g\precsim f$. We use $c$, $c_1$, $c_2$, etc. to denote some
universal constants (which might change in value each time they appear).

\subsection*{Related Work}

Research on the Expectation-Maximization (EM) algorithm and its convergence behavior in Gaussian mixture models has advanced rapidly. \citet{DBLP:journals/corr/BalakrishnanWY14} introduced a framework to delineate the region of convergence in terms of distribution parameters. Their work contrasted a population-level analysis with the sample-based implementation commonly used in practice, focusing specifically on the well-specified case of $k=2$ components and addressing both balanced and unbalanced scenarios.

Within this context, significant effort has been devoted to developing initialization strategies that guarantee convergence to the global optimum. \citet{klusowski2016statistical} demonstrated that local convergence can occur over a broader region than previously identified for the two-component case, while \citet{zhao2020statistical} investigated how initialization affects mixtures with an arbitrary number of well-separated components. In addition, \citet{daskalakis2017ten} provided global convergence guarantees for a two-component model with symmetrically positioned mean vectors. For mixtures with $k$ well-separated components, \citet{segol2021improved} proved that convergence is assured even when the algorithm is initialized near the midpoint between clusters, refining the estimation error bounds and extending the analysis to Gradient EM—a variant of the classical EM algorithm. Moreover, \citet{yan2017convergence} further analyzed the convergence rate and local contraction radius of Gradient EM for an arbitrary number of mixture components.

Another major line of inquiry has focused on model misspecification. \citet{dwivedi2018theoretical} examined an underspecified scenario, where a two-component Gaussian mixture is fitted to data generated by a three-component mixture, and characterized the resulting bias while also exploring the influence of initialization on convergence. The benefits of overspecified mixture models have been recognized by \citet{dwivedi2020singularity}, \citet{DBLP:conf/aistats/DwivediHKWJ020}, \citet{chen2024local}, and others. In particular, \citet{dwivedi2020singularity} and \citet{DBLP:conf/aistats/DwivediHKWJ020} studied the case of fitting two Gaussian components to data from a single Gaussian distribution. They compared balanced and unbalanced scenarios, demonstrating that in sample-based EM the unbalanced case converges at a statistical rate of $O(1/\sqrt{n})$, in contrast to $O(\sqrt[4]{1/n})$ for the balanced case when estimating mean vectors under both known and estimated isotropic covariance structures. They further showed that the algorithmic convergence rate is exponentially faster in the unbalanced setting.

Bayesian approaches to model overspecification, as discussed in \cite{rousseau}, have revealed that the estimated mixture weights can vary greatly, often causing some components to become redundant and allowing for model refinement by discarding those with very small weights. In addition, \citet{chen2024local} found that even spurious local minima of the negative log-likelihood retain structural information that is valuable for identifying component means, highlighting the advantages of overspecification over underspecification—a contrast often described as ``many-fit-one'' versus ``one-fit-many.'' Furthermore, \citet{dasgupta2013two} proposed a method for finite mixture overspecification by recommending that models be deliberately initialized with $\frac{\log(k)}{w_{\min}}$ clusters, where $w_{\min}$ denotes the smallest weight, to substantially accelerate convergence.

While much of the literature has focused on convergence in terms of distribution parameters, investigations measuring the quality of fit using the Kullback-Leibler (KL) divergence are relatively few. \citet{ghosal2001} derived a statistical convergence rate of $(\log n)^\kappa/\sqrt{n}$ in Hellinger distance, which translates to a lower bound of $(\log n)^{2\kappa}/n$ in KL divergence; however, their analysis was confined to well-specified models and did not consider algorithmic factors. \citet{dwivedi2018theoretical} also employed KL divergence in the context of underspecified mixtures, but, to our knowledge, the use of KL divergence in overspecified mixtures was first explored by \citet{xu2024toward}. They obtained KL divergence bounds for the population version of Gradient EM applied to a $k$-component mixture with known variances. In contrast, our work extends these results by analyzing both population and sample-based EM under a structured configuration of mixture centers and weights, and importantly, we establish an \emph{exponentially faster} algorithmic convergence rate in KL divergence than that reported by \citet{xu2024toward}.

\section{Initialization with $k$-means}\label{sec:kmeans}

Initialization is a critical step in the Expectation–Maximization (EM) algorithm, particularly in overspecified settings where the number of mixture components exceeds the true number. A common strategy is to first run the $k$–means algorithm (i.e., Lloyd's algorithm) on the data and then use the resulting cluster centers to initialize the EM algorithm.

When the data are generated from a single Gaussian distribution $\mathcal{N}(0,I)$, one observes that Lloyd's algorithm exhibits a natural fixed–point property under a symmetric configuration. In particular, consider initializing the $k$ centers at the vertices of a regular $(k-1)$–simplex centered at the origin. That is, let
\[
\mu_i = r\,v_i,\quad i=1,\dots,k,
\]
where the vectors $v_1,\dots,v_k\in \mathbb{R}^d$ (with $d\ge k-1$) are unit vectors forming the vertices of a regular simplex, and $r>0$ is a scaling factor.

A key observation is that the Voronoi partition induced by these centers depends only on the directions $v_i$ and not on the scalar $r$. Consequently, the conditional expectations computed in the Lloyd update—i.e., the new centers—are also determined solely by the angular configuration. In fact, one may show that the Lloyd update maps the configuration to
\[
\mu_i' = R_0\,v_i,\quad i=1,\dots,k,
\]
where $R_0>0$ is determined by the radial integrals of the Gaussian density. Thus, the fixed–point condition $\mu_i' = \mu_i$ for all $i$ is equivalent to choosing $r = R_0$.

This fixed–point property suggests that initializing the EM algorithm with a regular simplex (properly scaled) is natural in the context of overspecified Gaussian mixtures. In the proof of the following proposition (Section~\ref{app:kmeans_proof}), we rigorously analyze this phenomenon by first characterizing the Voronoi partition induced by a regular simplex and then proving that there exists a unique scaling $r>0$ such that the configuration
\[
\{\mu_i = r\,v_i :\, i=1,\dots,k\}
\]
remains invariant under the population–level Lloyd update.

\begin{proposition}\label{prop:kmeans}
Let $d\ge k-1$, and suppose that 
\[
v_1,\dots,v_k\in \mathbb{R}^d
\]
are unit vectors that form the vertices of a regular simplex in some $(k-1)$--dimensional  subspace of $\mathbb{R}^d$,
\[
\|v_i\|=1,\quad \text{for } i=1,\dots,k,\quad \text{with}\quad \sum_{i=1}^k v_i=0,
\]
with the pairwise inner products being constant for $i\neq j$. For any $r>0$, define centers
\[
\mu_i=r\,v_i,\quad i=1,\dots,k,
\]
and let the Voronoi cells be
\[
V_i=\{\, x\in\mathbb{R}^d : \|x-\mu_i\|\le \|x-\mu_j\| \text{ for all } j\neq i\,\}.
\]
Then there exists a unique $r>0$ such that if one performs the population-level Lloyd update
\[
\mu_i'=\frac{\displaystyle\int_{V_i} x\,\phi(x)\,dx}{\displaystyle\int_{V_i}\phi(x)\,dx},
\]
one obtains $\mu_i'=\mu_i$ for all $i=1,\dots,k$. That is, the configuration
\[
\{\mu_i=r\,v_i :\, i=1,\dots,k\}
\]
is a fixed point of Lloyd's algorithm.
\end{proposition}


\section{Population-Level Analysis} \label{sec:pop_level}

We begin by analyzing the behavior of the so-called \emph{population EM}, a theoretical construct that isolates algorithmic complexity from sample complexity. Population EM assumes direct access to the data-generating distribution \( \mathcal{N}({0},{I}) \) and, instead of maximizing the sample-based log-likelihood in \eqref{eq:mle}, optimizes the population log-likelihood:
\begin{equation}
\mathcal{L}(\theta):=\E_{{Z}\sim\mathcal{N}({0},{I})}[\log f({Z};\theta)].\label{eq:e_log_l}    
\end{equation}
The algorithm proceeds iteratively by applying the following two steps:

\begin{itemize}
    \item \emph{Expectation step}: Given the current estimate \( \theta_t \), compute the function
    \[
        Q(\theta,\theta_t):=\E_{Z\sim\mathcal{N}(0,I)}\left[\sum_{j=1}^k w_j(Z;\theta_t)\log\left(\pi_j\cdot\phi\left({Z-R^{j-1}\theta}\right)\right)\right],
    \]
    where 
    \[
    w_j(Z;\theta_t)=\frac{\pi_j\cdot\phi(Z-R^{j-1}\theta_t)}{\sum_{\ell=1}^k\pi_\ell\cdot\phi(Z-R^{\ell-1}\theta_t)}.
    \]
    
    \item \emph{Maximization step}: Update the parameters by solving the optimization problem:
    \[
    \theta_{t+1}\in\arg\max_{\theta}Q(\theta,\theta_t).
    \]
\end{itemize}

In this specific case, where the population EM algorithm is used to fit the mixture \eqref{eq:gmm} to \( \mathcal{N}(0,I) \), the recurrence relations governing the parameter updates can be explicitly derived (Section~\ref{app:EM_upd_proof}). The parameter updates follow the recursion \( \theta_{t+1} = M(\theta_t) \), where
    \begin{equation}
        M(\theta):=\E_{Z\sim\mathcal{N}(0,I)}\left[\sum_{j=1}^k w_j(Z;\theta) (R^{j-1})^\top Z\right].\label{eq:theta_upd}
    \end{equation}

The mapping \( M(\theta) \) is referred to as the \emph{population EM operator}. Notably, it is closely related to the population negative log-likelihood, as stated in the following equation (Section~\ref{app:NLL_M}):

\begin{equation}
\nabla_\theta[-\mathcal{L}(\theta)]=\theta-M(\theta).\label{eq:NLL_M}
\end{equation}

Denote
$L(\theta):=-\mathcal{L}(\theta)$. 
The equation \eqref{eq:NLL_M} implies that 
\[
\theta_{t+1} = \theta_t - \nabla_\theta[L(\theta_t)],
\]
which means that in the given setting, the EM algorithm is equivalent to gradient descent (GD) on $L(\theta)$ with a step size 1. This suggests that standard techniques used in the analysis of GD can be applied to study the convergence of the EM algorithm. One such technique is the Polyak-Łojasiewicz inequality, a sufficient condition for the exponential convergence of GD. We establish this property for \( L(\theta) \) in the following lemma.

\begin{lemma}[Local PL Inequality]
\label{lem:localPL}
Let $\mathcal{L}(\theta)$ be the population log-likelihood function defined by \eqref{eq:e_log_l}. Suppose $\pi_1, \ldots, \pi_k > 0$ are positive real numbers whose discrete Fourier transform has no zero entries. Then there exists \(\delta>0\) such that \(L(\theta):=-\mathcal{L}(\theta)\) satisfies the following local Polyak--Łojasiewicz (PL) inequality in \(\{\theta:\|\theta\|\le \delta\}\):
\begin{equation}
\,\bigl\|\nabla L(\theta)\bigr\|^2
\;\;\ge\;\;
\lambda_{\min} \,\Bigl(L(\theta)\;-\;L(0)\Bigr),\label{eq:pl_ineq}
\end{equation}
where $\lambda_{\min}\le1$ is the smallest eigenvalue of $\nabla^2L(0)$.
\end{lemma}

A key step in establishing the local Polyak--Łojasiewicz (PL) inequality around \(\theta^\ast=0\) is to show that the Hessian \(\nabla^2 L(\theta)\) remains positive definite in a sufficiently small neighborhood of \(\theta^\ast=0\).  Concretely, we need the Jacobian of the EM operator at \(\theta^\ast=0\) to have spectral properties that ensure strong convexity of the population negative log-likelihood \(L\).  

In our setup, this boils down to proving that the matrix \(
A \;=\; \sum_{j=1}^k \pi_j\,R^{j-1}
\) is invertible (cf.\ Lemmas~\ref{lem:M_jacobian} and \ref{lem:FourierInvertibility} in Section~\ref{app:local_pl}), since one can then show
\(
I - \frac{\partial M}{\partial\theta}(0) 
\;=\; 
A^{\!\top} A
\) is positive definite.  The invertibility of \(A\) follows from the assumption that the \(\pi_j\)'s have a discrete Fourier transform with no zero entries.  Intuitively, if the discrete Fourier transform of \(\{\pi_j\}\) vanished at one of the \(k\)-th roots of unity, then certain ``rotational symmetries'' in the update equations would cause degeneracies, preventing \(A\) from being invertible.  By ruling out such degeneracies, the condition \(\widehat{\pi}(\ell)\neq 0\) for all \(\ell\) guarantees the necessary full rank of \(A\).

Under these conditions, the Hessian \(\nabla^2 L(\theta)\) remains uniformly positive definite in a neighborhood of \(\theta^\ast=0\), which yields the strong convexity of \(L\) around \(\theta^\ast=0\).  From strong convexity, the usual argument then gives the local PL inequality \eqref{eq:pl_ineq}
demonstrating the sharpness of the landscape near the stationary point \(\theta^\ast=0\).

Since the local PL inequality plays a central role in our convergence analysis, we present the proof of Lemma~\ref{lem:localPL} in the main text (Section~\ref{app:local_pl}) to ensure the core argument remains transparent. The proofs of the remaining supporting lemmas are deferred to the Appendix.

We are ready to prove the exponential decay of the KL divergence between the true distribution and the sequence of fitted mixtures.

\vspace{10pt}

\begin{proof}[Proof of Theorem~\ref{thm:population}]
We start by noting that
\[
\kl\bigl[\mathcal{N}(0,I)\,\bigl\|\,\mathcal{G}(\theta_t)\bigr]
\;=\;
L(\theta_t)\;-\;L(0).
\]
Equation~\eqref{eq:NLL_M} implies that the Hessian of \(L\) is given by
\[
\nabla^2 L(\theta) \;=\; I \;-\;\frac{\partial M}{\partial \theta}.
\]
Furthermore, we can show (see Lemma~\ref{lem:M_jacobian} in Section~\ref{app:local_pl}) that at \(\theta^* = 0\), we have
\[
\nabla^2 L(0) \;=\; A A^{\!\top},
\quad
\text{where}
\quad
A \;:=\;\sum_{j=1}^k \pi_j \,R^{j-1}.
\]
Since \(R\) is an orthogonal matrix, its eigenvalues are among the \(k\)-th roots of unity 
\(\{e^{2\pi i \ell/k}\}_{\ell=0}^{k-1}\), implying \(\|R\|_{\mathrm{op}} = 1\).  
Hence, by the triangle inequality,
\[
\|\nabla^2 L(0)\|_{\mathrm{op}}
\;\le\;
\biggl(\sum_{j=1}^k \pi_j \|R^{j-1}\|_{\mathrm{op}}\biggr)^2
\;=\;
\biggl(\sum_{j=1}^k \pi_j\biggr)^2
\;=\;
1.
\]
By smoothness of $L(\theta)$, there is therefore a neighborhood of \(\theta^* = 0\) in which 
\(\|\nabla^2 L(\theta)\|_{\mathrm{op}} \le 3/2\). Consequently, for 
\(\theta\) and \(\theta'\) in that neighborhood,
\[
L(\theta') 
\;\le\; 
L(\theta) \;+\;\nabla L(\theta)^\top(\theta' - \theta)
\;+\;\tfrac{3}{4}\,\|\theta' - \theta\|^2.
\]
In particular, applying this to \(\theta_{t+1}\) and \(\theta_t\) yields
\[
\begin{aligned}
L(\theta_{t+1})
&\;\le\;
L(\theta_t) 
\;+\;\nabla L(\theta_t)^\top \bigl(\theta_{t+1} - \theta_t\bigr)
\;+\;\tfrac{3}{4}\,\|\theta_{t+1} - \theta_t\|^2 \\[6pt]
&\;=\;
L(\theta_t) 
\;+\;\nabla L(\theta_t)^\top \bigl(M(\theta_t) - \theta_t\bigr)
\;+\;\tfrac{3}{4}\,\|M(\theta_t) - \theta_t\|^2 \\[6pt]
&\;=\;
L(\theta_t) 
\;-\;\|\nabla L(\theta_t)\|^2
\;+\;\tfrac{3}{4}\,\|\nabla L(\theta_t)\|^2 \\[3pt]
&\;=\;
L(\theta_t)
\;-\;\tfrac{1}{4}\,\|\nabla L(\theta_t)\|^2.
\end{aligned}
\]
Next, using the Polyak--Lojasiewicz inequality~\eqref{eq:pl_ineq}, we obtain
\[
L(\theta_{t+1})  
\;\le\;
L(\theta_t)-\;\tfrac{\lambda_{\min}}{4}\,\bigl(L(\theta_t) - L(0)\bigr).
\]
Subtracting \(L(0)\) from both sides gives
\[
L(\theta_{t+1}) - L(0)
\;\le\;
\Bigl(1 - \tfrac{\lambda_{\min}}{4}\Bigr)\,\bigl(L(\theta_t) - L(0)\bigr).
\]
Applying this inequality recursively completes the proof.
\end{proof}

\section{Finite-Sample Analysis}\label{sec:finite_sample}

When the sample-based averaged log-likelihood in \eqref{eq:mle} is maximized via the EM algorithm, the parameter updates can be expressed explicitly by replacing the expectation \(\mathbb{E}\) in \eqref{eq:theta_upd} with the empirical average over the sample:
\begin{align}
    &\hat{\theta}_{t+1} = M_n(\hat{\theta}_t), \notag \\
    \text{where} \quad &M_n(\theta) := \frac{1}{n} \sum_{i=1}^{n} \sum_{j=1}^{k} w_j(Z_i; \theta) (R^{j-1})^\top Z_i.
\end{align}
The following perturbation bound (Section~\ref{app:perturb_bound}) relates the sample-based EM operator to its population-level counterpart:
\begin{equation}
    \Pr\left[ \sup_{\|\theta\| \le r}
    \left\| M_n(\theta) - M(\theta) \right\|
    \le c r \sqrt{\frac{d + \log(1/\delta)}{n}} \right] \ge 1 - \delta,
\end{equation}
for any radius \(r > 0\), threshold \(\delta \in (0,1)\), and sufficiently large \(n\).

Due to the strict contractivity of the population EM operator in a neighborhood of \(\theta^\ast = 0\) (Lemma~\ref{lem:PositiveDefinite} in Section~\ref{app:local_pl}) and the perturbation bound above, we can establish that the sequence of EM iterates \(\hat{\theta}_t\) satisfies, with probability at least \(1 - \delta\),
\begin{equation}
    \|\hat{\theta}_T\| \precsim \|\theta_0\| \sqrt{\frac{\log(1/\delta)}{n}}, \label{eq:theta_bound}
\end{equation}
for \( T \succsim \log\left(\frac{n}{\log(1/\delta)}\right) \), provided that \(\theta_0\) lies within the contraction neighborhood (see the proof of Theorem~2 in \cite{DBLP:journals/corr/BalakrishnanWY14}).

With this, we are ready to establish our key result on convergence in KL distance for the finite-sample case.

\vspace{10pt}

\begin{proof}[Proof of Theorem~\ref{thm:main}]
    By the convexity of \(L\) in a neighborhood of \(\theta^\ast = 0\) (Lemma~\ref{lem:local-strong-convexity} in Section~\ref{app:local_pl}), we have
    \begin{equation}
        L(\hat{\theta}_t) - L(0) \leq \nabla L(\hat{\theta}_t)^\top \hat{\theta}_t. \label{eq:L_diff}
    \end{equation}
    From \eqref{eq:NLL_M}, it follows that
    \begin{equation}
        \nabla L(\hat{\theta}_t)^\top \hat{\theta}_t = \|\hat{\theta}_t\|^2 - [M(\hat{\theta}_t)]^\top \hat{\theta}_t \precsim \|\hat{\theta}_t\|^2, \label{eq:L_diff_bound}
    \end{equation}
    where we used the contraction property of \(M\) near \(\theta = 0\) (Corollary~\ref{cor:contraction} in Section~\ref{app:local_pl}). 

    The theorem follows directly from \eqref{eq:theta_bound}, \eqref{eq:L_diff}, and \eqref{eq:L_diff_bound}.
\end{proof}

\section{Proof of the Local PL Inequality}\label{app:local_pl}\label{app:proofs}

In this section, we prove the local Polyak--Lojasiewicz (PL) inequality for the negative log-likelihood \(L(\theta)\) of our overspecified Gaussian mixture model. By analyzing the Jacobian of the population EM operator \(M(\theta)\) at \(\theta^* = 0\), we show that \(M(\theta)\) is locally contractive, which implies that the Hessian of \(L(\theta)\) is uniformly positive definite near \(\theta^* = 0\).

The subsequent lemmas establish these properties and lead directly to the local PL inequality, ensuring the exponential convergence of the EM algorithm in terms of the KL divergence.

\begin{lemma}\label{lem:M_jacobian}
    Let $M(\theta)$ be the EM operator defined by \eqref{eq:theta_upd}. Then the Jacobian of $M(\theta)$ at $\theta^\ast=0$ is given by
    \(
  \frac{\partial M}{\partial\theta}(0)
  \;=\;
  I
  \;-\;
  \biggl(\sum_{j=1}^k \pi_j\,R^{j-1}\biggr)\,
  \biggl(\sum_{j=1}^k \pi_j\,R^{j-1}\biggr)^\top.
\)
\end{lemma}
\begin{proof}
Let \(
S(\theta,Z)
\;=\;
\sum_{\ell=1}^k
\pi_\ell\,\exp\!\bigl((R^{\ell-1}\theta)^\top Z\bigr)
\). Then
\[
w_j(Z;\theta)
=
\frac{\pi_j\,\exp\!\bigl((R^{j-1}\theta)^\top Z\bigr)}{S(\theta,Z)}.\] At \(\theta^\ast=0\), each exponential term is \(\exp(0)=1\), so
\(
S(0,Z)
\;=\;
\sum_{\ell=1}^k \pi_\ell
\;=\;1\), \(w_j(Z;0)
\;=\;\pi_j\). Since
\begin{align*}
\nabla_\theta\,\bigl((R^{j-1}\theta)^\top Z\bigr)
\;&=\;
(R^{j-1})^\top Z,\\
\nabla_\theta S(\theta,Z)
\;&=\;
\sum_{\ell=1}^k
\pi_\ell\,\exp\!\bigl((R^{\ell-1}\theta)^\top Z\bigr)\,\bigl[(R^{\ell-1})^\top Z\bigr],
\end{align*}
using the quotient rule for \(\nabla_\theta w_j(Z;\theta)\), and then evaluating it at \(\theta^\ast=0\), we get
\[
\bigl.\nabla_\theta w_j(Z;\theta)\bigr|_{\theta^\ast=0}
\;=\;
\pi_j
\Bigl[
  (R^{j-1})^\top Z
  \;-\;
  \sum_{\ell=1}^k \pi_\ell\,(R^{\ell-1})^\top Z
\Bigr].
\]
Define
\(
g(\theta,Z)
\;=\;
\sum_{j=1}^k w_j(Z;\theta)\,(R^{j-1})^\top Z
\). Then
\[
\frac{\partial g}{\partial\theta}(\theta,Z)
\;=\;
\sum_{j=1}^k (R^{j-1})^\top Z
\,\bigl[\nabla_\theta w_j(Z;\theta)\bigr]^\top.\] At \(\theta^\ast=0\),
\[
\bigl.\frac{\partial g}{\partial\theta} (\theta,Z)\bigr|_{\theta^\ast=0}
\;=\;
\sum_{j=1}^k 
\pi_j
(R^{j-1})^\top Z\,\biggl[
  (R^{j-1})^\top Z
  \;-\;
  \sum_{\ell=1}^k \pi_\ell\,(R^{\ell-1})^\top Z
\biggr]^\top.
\]
Then the sought Jacobian of \(M(\theta)\) at \(\theta^\ast=0\) is
\[
\frac{\partial M}{\partial\theta} (0)
\;=\;
\mathbb{E}_Z\bigl[\nabla_\theta g(0,Z)\bigr].
\]
Since \(Z\sim \mathcal{N}_d(0,I)\), we have \(\mathbb{E}[Z Z^\top] = I\). Each \(R^{j-1}\) is orthogonal, hence
\[
\mathbb{E}\!\Bigl[(R^{j-1})^\top Z\,Z^\top\,R^{j-1}\Bigr]
\;=\;
(R^{j-1})^\top\,I\,R^{j-1}
\;=\;
I.
\]
Collecting terms, the result is 
\[
\frac{\partial M}{\partial\theta}(0)
\;=\;
I
\;-\;
\biggl(\sum_{j=1}^k \pi_j\,R^{j-1}\biggr)
\biggl(\sum_{j=1}^k \pi_j\,R^{j-1}\biggr)^\top,
\]
which completes the proof.
\end{proof}

\begin{lemma}\label{lem:FourierInvertibility}
Let  
$R \in \mathbb{R}^{d \times d}$ be a (real) matrix whose eigenvalues lie among 
the $k$-th roots of unity (except $1$), i.e.,
\(
\mathrm{spec}(R)
\;\subseteq\;
\bigl\{
e^{\,i \frac{2\pi \ell}{k}} : \ell = 1, 2, \ldots, k-1
\bigr\}.
\)
Let $\pi_1, \ldots, \pi_k > 0$ be positive real numbers whose discrete Fourier transform
\(
\widehat{\pi}(\ell) 
\;=\; 
\sum_{j=0}^{k-1} 
\pi_{j+1}
\,e^{\,i \tfrac{2\pi \ell}{k}j}\), 
\(\ell = 0, 1, \dots, k-1,
\)
has no zero entries (i.e.\ $\widehat{\pi}(\ell) \neq 0$ for all $\ell$).
Define the matrix
\(
A 
\;:=\; 
\sum_{j=1}^k 
\pi_j \, R^{\,j-1}.
\)
Then $A$ is invertible.
\end{lemma}

\begin{proof}
Since all eigenvalues of $R$ lie among the $k$-th roots of unity (except  $1$), 
we can work over $\mathbb{C}$ and bring $R$ into a Jordan (or block-diagonal) form. 
Concretely, there exists an invertible matrix $V \in \mathbb{C}^{d \times d}$ such that \(
R 
\;=\; 
V\,\Lambda\,V^{-1},
\)
where $\Lambda$ is block-diagonal and each block corresponds to an eigenvalue 
of the form $e^{\,i \tfrac{2\pi \ell}{k}}$ (with $1 \le \ell \le k-1$).  In particular,
\(
R^{\,j-1}
\;=\;
V \,\Lambda^{\,j-1}\,V^{-1}\) for all 
\(
j=1,\dots,k
\). Thus we can rewrite $A$ as
\[
A 
\;=\; 
\sum_{j=1}^k 
\pi_j \, R^{\,j-1}
\;=\;
\sum_{j=1}^k 
\pi_j \bigl(V\,\Lambda^{\,j-1}\,V^{-1}\bigr)
\;=\;
V
\Bigl(
\sum_{j=1}^k 
\pi_j \,\Lambda^{\,j-1}
\Bigr)
V^{-1}.
\]
Since $V$ is invertible, $A$ is invertible if and only if 
\(
\sum_{j=1}^k \pi_j \,\Lambda^{\,j-1}
\)
is invertible.

Now, $\Lambda$ is block-diagonal with Jordan blocks corresponding to eigenvalues 
$\lambda \in \{\,e^{\,i\tfrac{2\pi \ell}{k}} : \ell=1,\dots,k-1\}\,$. 
Consider a single eigenvalue $\lambda$.  The diagonal entry of the diagonal block of 
\(\sum_{j=1}^k \pi_j \,\Lambda^{\,j-1}\)
is 
\(
\sum_{j=1}^k 
\pi_j\,\lambda^{\,j-1}
\;=\;
\sum_{j=0}^{k-1}
\pi_{j+1}\,\lambda^{\,j}.
\)
Since $\lambda^{\,j} = e^{\,i \tfrac{2\pi \ell}{k} j}$ for some $\ell \in \{1,\dots,k-1\}$, 
this sum is precisely the  discrete Fourier transform of $(\pi_1,\dots,\pi_k)$:
\(
\sum_{j=0}^{k-1}
\pi_{j+1}\,e^{\,i \tfrac{2\pi \ell}{k} j}
\;=\;
\widehat{\pi}(\ell).
\)
By hypothesis, $\widehat{\pi}(\ell) \neq 0$ for all $\ell=0,\dots,k-1$, hence each diagonal 
entry is nonzero. Therefore, every diagonal block of 
\(\displaystyle \sum_{j=1}^k \pi_j\,\Lambda^{\,j-1}\)
is invertible, so the entire block-diagonal matrix is invertible.  
\end{proof}

\begin{lemma}\label{lem:PositiveDefinite}
Under the conditions of Lemmas \ref{lem:M_jacobian} and \ref{lem:FourierInvertibility}, the matrix 
\(
I - \frac{\partial M}{\partial\theta}(0)
\)
is positive definite.
\end{lemma}

\begin{proof}
From Lemma \ref{lem:M_jacobian}, the Jacobian of the EM operator at $\theta^\ast=0$ is given by
\(
\frac{\partial M}{\partial\theta}(0) = I - A A^{\!\top}\), 
where 
\(
A = \sum_{j=1}^k \pi_j R^{j-1}\). 
Rearranging this equation, we obtain
\(
I - \frac{\partial M}{\partial\theta}(0) = A A^{\!\top}\). By Lemma \ref{lem:FourierInvertibility}, the matrix \( A \) is invertible. Since \( A A^{\!\top} \) is the product of \( A \) and \( A^{\!\top} \), it follows that \( A A^{\!\top} \) is symmetric and positive definite. To see this, note that for any non-zero vector \( x \in \mathbb{R}^d \),
\(
x^{\!\top} A A^{\!\top} x = \left( A^{\!\top} x \right)^{\!\top} \left( A^{\!\top} x \right) = \| A^{\!\top} x \|^2 > 0\), 
since \( A \) is invertible and thus \( A^{\!\top} x \neq 0 \) for \( x \neq 0 \). 
Therefore, the matrix 
\(
I - \frac{\partial M}{\partial\theta}(0) = A A^{\!\top}
\)
is positive definite.
\end{proof}
\begin{corollary}\label{cor:contraction}
    All eigenvalues of $\frac{\partial M}{\partial\theta}(0)$ lie strictly below 1, which in turn implies that $M$ is a contraction near $\theta^\ast=0$.
\end{corollary}

\begin{lemma}\label{lem:local-strong-convexity} Let $\mathcal{L}(\theta)$ be the population log-likelihood function defined by \eqref{eq:e_log_l}. Suppose $\pi_1, \ldots, \pi_k > 0$ are positive real numbers whose discrete Fourier transform has no zero entries. Then there exists \(\delta>0\) such that \(L\) is strongly convex in \(\{\theta:\|\theta\|\le \delta\}\):
\end{lemma}

\begin{proof}
Since $\nabla L(\theta) = \theta - M(\theta)$, its Hessian is given by
\(
\nabla^2 L(\theta)
\;=\;
I \;-\; \frac{\partial M}{\partial \theta}(\theta)\). 
By Lemma~\ref{lem:PositiveDefinite},
\(
S 
\;:=\;
I \;-\;
\frac{\partial M}{\partial \theta}(0)
\)
is positive definite.  Let 
\(
\lambda_{\min} 
\;=\; 
\lambda_{\min}(S)
\;>\;
0
\) 
be the smallest eigenvalue of $S$.  

Next, by continuity of $\tfrac{\partial M}{\partial \theta}(\theta)$ at $\theta = 0$, for any $\varepsilon > 0$ there exists a $\delta>0$ such that
\[
\|\theta\| < \delta
\quad\Longrightarrow\quad
\left\|
  \frac{\partial M}{\partial \theta}(\theta)
  \;-\;
  \frac{\partial M}{\partial \theta}(0)
\right\|_{\mathrm{op}}
< \varepsilon,
\]
where $\|\cdot\|_{\mathrm{op}}$ denotes the operator norm.  

Choose 
\(
\varepsilon \;=\; \tfrac{1}{2}\,\lambda_{\min}\). Then for $\|\theta\| < \delta$,
\(
\Bigl\|
  \frac{\partial M}{\partial \theta}(\theta)
  \;-\;
  \frac{\partial M}{\partial \theta}(0)
\Bigr\|_{\mathrm{op}}
< \frac{1}{2}\,\lambda_{\min}\). 
Thus, for any vector $v \in \mathbb{R}^d$ with $\|v\|=1$,
\[
\begin{aligned}
v^\top \Bigl(I - \frac{\partial M}{\partial \theta}(\theta)\Bigr) v
&=\;
v^\top \bigl(I - \tfrac{\partial M}{\partial \theta}(0)\bigr)v
\;-\;
v^\top \Bigl(\tfrac{\partial M}{\partial \theta}(\theta)
            - 
            \tfrac{\partial M}{\partial \theta}(0)\Bigr)v
\\[6pt]
&\ge\;
v^\top \Bigl(I - \tfrac{\partial M}{\partial \theta}(0)\Bigr)v
\;-\;
\Bigl\|\tfrac{\partial M}{\partial \theta}(\theta)
       - 
       \tfrac{\partial M}{\partial \theta}(0)\Bigr\|_{\mathrm{op}}
\\[6pt]
&\ge\;
\lambda_{\min}
\;-\;
\tfrac12\,\lambda_{\min}
\;=\;
\tfrac12\,\lambda_{\min}.
\end{aligned}
\]
Therefore,
\(
I \;-\;
\frac{\partial M}{\partial \theta}(\theta)
\;\succeq\;
\frac12\,\lambda_{\min} \, I
\quad
\text{for all } \|\theta\|<\delta.
\)
Since $\nabla^2 L(\theta) = I - \frac{\partial M}{\partial \theta}(\theta)$, we deduce
\(
\nabla^2 L(\theta)
\;\succeq\;
\frac12\,\lambda_{\min}\,I
\quad
\text{whenever }
\|\theta\|<\delta.
\)
Hence $L(\theta)$ is $\bigl(\tfrac12\,\lambda_{\min}\bigr)$-strongly convex in the ball 
\(\{\theta : \|\theta\| < \delta\}\). 
\end{proof}

\begin{proof}[Proof of Lemma~\ref{lem:localPL}]
By Lemma~\ref{lem:local-strong-convexity}, $L$ is $\frac{\lambda_{\min}}2$-strongly convex in a neighborhood of $\theta^\ast=0$, i.e. there exists $\delta>0$ such that for $\theta,\theta'\in\{\theta:\|\theta\|\le\delta\}$

$$
L(\theta')\ge L(\theta)+\nabla L(\theta)^T(\theta'-\theta)+\frac{\lambda_{\min}}{4}\| \theta'-\theta \|^2.
$$
Minimizing both sides with respect to $\theta'$, we get
$$
L(0)\ge L(\theta) - \frac{1}{\lambda_{\min}}\|\nabla L(\theta)\|^2.
$$
Re-arranging the terms we have the PL inequality.
\end{proof}

\begin{credits}
\subsubsection{\ackname} This research has been funded by the Science Committee of the Ministry of Science and Higher Education of the Republic of Kazakhstan (Grant No. AP27510283). Artur Pak's work was supported by Nazarbayev University under Faculty-development competitive research grants program for 2023-2025 Grant \#20122022FD4131, PI R. Takhanov. Zhenisbek Assylbekov's work was supported by Purdue University Fort Wayne under Summer Research Grant Program 2024, and he would like to thank Igor Melnykov and Francesco Sica for useful discussions.

\subsubsection{\discintname}
The authors have no competing interests to declare that are
relevant to the content of this article.
\end{credits}

\bibliographystyle{splncs04nat}
\bibliography{ref}

\appendix

\section{Proof of Proposition~\ref{prop:kmeans}}\label{app:kmeans_proof}

\begin{proof}
We prove the proposition in several steps.

\medskip

\paragraph{Step 1. Independence of the Voronoi Partition from $r$.} 
For any two indices $i\neq j$, consider the condition that a point $x\in\mathbb{R}^d$ is closer to $\mu_i$ than to $\mu_j$. We have
\[
\|x-\mu_i\|^2\le \|x-\mu_j\|^2.
\]
Expanding, we obtain
\[
\|x\|^2 - 2r\,x\cdot v_i + r^2\le \|x\|^2 - 2r\,x\cdot v_j + r^2,
\]
which simplifies to
\[
x\cdot (v_i-v_j)\ge 0.
\]
Since this inequality depends only on $x$ and the fixed unit vectors $v_i$, the Voronoi cell
\begin{equation}
V_i=\{\,x\in\mathbb{R}^d: \, x\cdot (v_i-v_j)\ge 0\text{ for all }j\neq i\,\}\label{eq:voronoi}
\end{equation}
is independent of $r$.

\medskip

\paragraph{Step 2. Expression for the Lloyd Update.} 
Write any $x\in\mathbb{R}^d$ in spherical coordinates:
\[
x=\rho\,u,\quad \rho\ge0,\quad u\in S^{d-1},
\]
with the volume element $dx=\rho^{\,d-1}d\rho\,d\sigma(u)$, where $d\sigma(u)$ is the surface measure on the unit sphere $S^{d-1}$. Then
\[
\phi(x)=\frac{1}{(2\pi)^{d/2}} e^{-\rho^2/2}.
\]
By \eqref{eq:voronoi}, there exists a subset $U_i\subset S^{d-1}$ such that
\[
V_i=\{\rho\,u: \, \rho\ge0,\ u\in U_i\,\}.
\]
Then the denominator in the update for $\mu_i$ is
\[
\int_{V_i}\phi(x)\,dx = \int_{U_i}\!d\sigma(u) \int_{0}^{\infty} \rho^{\,d-1} \frac{1}{(2\pi)^{d/2}}e^{-\rho^2/2}\,d\rho,
\]
and the numerator is
\[
\int_{V_i} x\,\phi(x)\,dx = \int_{U_i} u\,d\sigma(u) \int_{0}^{\infty} \rho^{\,d} \frac{1}{(2\pi)^{d/2}}e^{-\rho^2/2}\,d\rho.
\]
Define the radial factor
\[
R_0:=\frac{\displaystyle\int_{0}^{\infty} \rho^{\,d} e^{-\rho^2/2}\,d\rho}{\displaystyle\int_{0}^{\infty} \rho^{\,d-1} e^{-\rho^2/2}\,d\rho}.
\]
Then the Lloyd update yields
\[
\mu_i' = R_0\, \frac{\displaystyle\int_{U_i} u\,d\sigma(u)}{\displaystyle\int_{U_i} d\sigma(u)}.
\]

\medskip

\paragraph{Step 3. Symmetry of the Angular Integration.} 
Since the vectors $v_1,\dots,v_k$ form a regular simplex in some $(k-1)$--dimensional subspace, the average
\[
\frac{\displaystyle\int_{U_i} u\,d\sigma(u)}{\displaystyle\int_{U_i} d\sigma(u)}
\]
is exactly $v_i$ by symmetry. Consequently,
\[
\mu_i' = R_0\, v_i.
\]

\medskip

\paragraph{Step 4. Fixed-Point Condition.} 
Our original centers were $\mu_i=r\,v_i$. For the configuration to be fixed under the Lloyd update, we require
\[
\mu_i'=\mu_i,\quad \text{for all } i,
\]
which is equivalent to
\[
R_0\,v_i = r\,v_i,\quad \text{for all } i.
\]
Since the $v_i$ are nonzero, this reduces to
\[
r=R_0.
\]
Notice that $R_0>0$ depends only on the dimension $d$ (through the radial integrals) and is independent of the particular value of $r$. Thus, there exists a unique $r>0$ (namely, $r=R_0$) for which the centers
\[
\mu_i=r\,v_i,\quad i=1,\dots,k,
\]
are fixed by the Lloyd update.

\end{proof}

\section{Population EM properties}

\subsection{Population EM updates}\label{app:EM_upd_proof}

We begin by providing additional details on the EM algorithm. It is convenient to represent the mixture distribution \eqref{eq:gmm} using a latent categorical random variable \( K \), which identifies the mixture components. Given the mixture weights \( (\pi_1,\ldots,\pi_k) \), we assume that 
\[
\Pr[K=j]=\pi_j.
\]
The conditional distribution of \( X \) given \( K=j \) is then defined as
\[
(X \mid K=j) \sim \mathcal{N}_d(R^{j-1}\theta, I), \quad \text{for } j \in [k].
\]
This specifies the joint distribution of the tuple \( (X, K) \), ensuring that the marginal distribution of \( X \) corresponds to the Gaussian mixture \( \mathcal{G}(\theta) \) in \eqref{eq:gmm}. The Population EM algorithm maximizes the expected log-likelihood \eqref{eq:e_log_l} through the following iterative steps:

\begin{itemize}
    \item \emph{E-step:} Given the current estimate \( \theta_t \), compute the soft assignment of any \( x \in \mathbb{R}^d \) to component \( K=j \), i.e., evaluate the posterior probability:
    \begin{equation}
    w_j(x;\theta_t)=\frac{\pi_j\cdot\phi(x-R^{j-1}\theta_t)}{\sum_{\ell=1}^k\pi_\ell\cdot\phi(x-R^{\ell-1}\theta_t)},\label{eq:weight_fn}  
    \end{equation}
    and use it to compute the \( Q \)-function:
    \begin{equation*}
        Q(\theta,\theta_t)
        :=\mathbb{E}_{Z\sim\mathcal{N}(0,I)}\left[\sum_{j=1}^k w_j(Z;\theta_t)\log \left(\pi_j\cdot \phi(Z-R^{j-1}\theta)\right)\right].
    \end{equation*}
   
    \item \emph{M-step:} Update the parameter by solving the following optimization problem:
    \begin{equation}
        \theta_{t+1} \in \arg\max_{\theta} Q(\theta,\theta_t).\label{eq:m_step}
    \end{equation}
\end{itemize}

\begin{lemma} \label{lem:EM_upd}
    Let the population EM algorithm maximize the expected log-likelihood \eqref{eq:e_log_l}. Then, the parameter updates follow the recursion \( \theta_{t+1} = M(\theta_t) \), where
    $$
        M(\theta):=\E_{Z\sim\mathcal{N}(0,I)}\left[\frac{\sum_{j=1}^k\pi_j\cdot\exp\left((R^{j-1}\theta)^\top Z\right) (R^{j-1})^\top Z}{\sum_{\ell=1}^k \pi_\ell\cdot\exp\left((R^{\ell-1}\theta)^\top Z\right)}\right].
    $$
\end{lemma}

\begin{proof}
    First, we rewrite the weight function \eqref{eq:weight_fn} as
    \[
    w_j(x;\theta)=\frac{\pi_j\cdot\exp\left((R^{j-1}\theta)^\top x\right)}{\sum_{\ell=1}^k\pi_\ell\cdot\exp\left((R^{\ell-1}\theta)^\top x\right)}.
    \]
    The gradient of \( Q \) with respect to \( \theta \) is given by
    \begin{align*}
    \nabla_{\theta} Q &=\mathbb{E}_{Z}\left[\sum_{j=1}^k w_j(Z;\theta_t)(R^{j-1})^\top(Z-R^{j-1}\theta)\right] \\
    &=\mathbb{E}_{Z}\left[\sum_{j=1}^k w_j(Z;\theta_t)(R^{j-1})^\top Z\right]-\mathbb{E}_{Z}\left[\sum_{j=1}^k w_j(Z;\theta_t)\right]\cdot\theta \\
    &=\mathbb{E}_{Z}\left[\sum_{j=1}^k w_j(Z;\theta_t)(R^{j-1})^\top Z\right]-\theta,
    \end{align*}
    where we used the orthogonality of \( R \) and the fact that \( \sum_{j=1}^k w_j(Z;\theta_t) = 1 \).

    Setting \( \nabla_{\theta} Q = 0 \) and solving for \( \theta \), we obtain the population EM update for the location parameter:
    \begin{equation}
        \theta_{t+1} =\mathbb{E}_Z\left[\sum_{j=1}^k w_j(Z;\theta_t)(R^{j-1})^\top Z\right]
        \label{eq:theta_upd_raw}
    \end{equation}
    which completes the proof.
\end{proof}

\subsection{Population EM operator and Log-likelihood}\label{app:NLL_M}
\begin{lemma}\label{lem:NLL_M} 
The population log-likelihood \( \mathcal{L}(\theta) \) defined in \eqref{eq:e_log_l} and the population EM operator \( M(\theta) \) are related by the equation:
    $$
    \nabla_\theta[-\mathcal{L}(\theta)]=\theta-M(\theta).
    $$
\end{lemma}
\begin{proof}
    Since \( f({x};\theta) \) is the probability density function (p.d.f.) of the Gaussian mixture \eqref{eq:gmm}, it can be expressed as
    \begin{align*}
    f({x};\theta)&=\sum_{j=1}^k\pi_j\cdot\phi\left({{x}-R^{j-1}\theta}\right)\\
    &=(2\uppi)^{-d/2} \cdot \exp\left(-\frac{\|{x}\|^2+\|\theta\|^2}{2}\right) \cdot \sum_{j=1}^k\pi_j\cdot\exp\left((R^{j-1}\theta)^\top x\right).
    \end{align*}
    Taking the logarithm, we obtain
    \[
        \log f({x};\theta) = -\frac{d}{2} \log(2\uppi) - \frac{\|{x}\|^2 + \|\theta\|^2}{2} + \log\left(\sum_{j=1}^k\pi_j\cdot\exp\left((R^{j-1}\theta)^\top x\right)\right).
    \]
    Thus, the negative population log-likelihood is given by
    \begin{align}
        -\mathcal{L}(\theta) &= -\mathbb{E}_{Z}\left[\log f(Z;\theta)\right] \notag\\
        &= \frac{d}{2} \log(2\uppi) + \frac{d + \|\theta\|^2}{2} - \mathbb{E}_{Z}\left[\log\left(\sum_{j=1}^k\pi_j\cdot\exp\left((R^{j-1}\theta)^\top Z\right)\right)\right].\notag
    \end{align}  
    Differentiating both sides, we obtain the gradient of \( -\mathcal{L}(\theta) \):
    \begin{align*}
    \nabla_\theta[-\mathcal{L}(\theta)] &= \theta - \mathbb{E}_{Z}\left[\frac{\sum_{j=1}^k\pi_j\cdot\exp\left((R^{j-1}\theta)^\top Z\right)(R^{j-1})^\top Z}{\sum_{j=1}^k\pi_j\cdot\exp\left((R^{j-1}\theta)^\top Z\right)}\right]\\
    &= \theta - M(\theta),
    \end{align*}
    which completes the proof.
\end{proof}

\section{Perturbation bound}\label{app:perturb_bound}
\begin{lemma}
\label{lem:perturb}
There exist universal constants $c,c' > 0$ such that for any radius $r>0$, confidence level $\delta \in (0,1)$, and sample size $n \ge c'\,\bigl(d + \log(1/\delta)\bigr)$, the following holds with probability at least $1-\delta$:
\[
\sup_{\|\theta\|\le r}
\left\|\,M_{n}(\theta)\;-\;M(\theta)\right\|
\;\;\le\;
c r\sqrt{\frac{d + \log(1/\delta)}{n}}.
\]
\end{lemma}

\begin{proof}
Let $\mathcal{S}^d := \{u \in \mathbb{R}^d : \|u\|_2 = 1\}$ be the unit sphere, and define
\[
Z
\;:=\;
\sup_{\|\theta\|\le r}
\,\Bigl\|\,M_{n}(\theta)\;-\;M(\theta)\Bigr\|_2
\;=\;
\sup_{\|\theta\|\le r}
\,\sup_{u \,\in\, \mathcal{S}^d}
\,\bigl(M_{n}(\theta) - M(\theta)\bigr)^\top u.
\]
By a standard covering argument (e.g.\ Chapter~5 of~\cite{Wainwright_2019}), there is a $1/8$-net $\{u_1,\dots,u_N\}\subset \mathcal{S}^d$ with $N \le 17^d$ such that
\[
Z 
\;\;\le\;\;
\frac{8}{7}\;\max_{1\le j\le N}
\sup_{\|\theta\|\le r}
\,\bigl(M_{n}(\theta) - M(\theta)\bigr)^\top u_j.
\]
Hence it suffices to control
\[
Z_{u_j}
\;:=\;
\sup_{\|\theta\|\le r}
\,\bigl(M_{n}(\theta) - M(\theta)\bigr)^\top u_j,
\quad
\text{for each } j=1,\dots,N.
\]
Notice that
\begin{align*}
&M_{n}(\theta) - M(\theta) =\frac1n\sum_{i=1}^n \left[ \sum_{\ell=1}^k w_\ell(Z_i;\theta)(R^{\ell-1})^\top Z_i\right]
\;-\;
\mathbb{E}\left[\sum_{\ell=1}^k w_\ell(Z;\theta)(R^{\ell-1})^\top Z\right]\\
&=\frac1n\sum_{i=1}^n \left[ \sum_{\ell=1}^k \pi_\ell(R^{\ell-1})^\top Z_i\right]
-
\mathbb{E}\left[\sum_{\ell=1}^k \pi_\ell(R^{\ell-1})^\top Z\right]\\
&\quad+\frac1n\sum_{i=1}^n\left[  \sum_{\ell=1}^k (w_\ell(Z_i;\theta)-\pi_\ell)(R^{\ell-1})^\top Z_i\right]
\;-\;
\mathbb{E}\left[\sum_{\ell=1}^k (w_\ell(Z;\theta)-\pi_\ell)(R^{\ell-1})^\top Z\right]\\
\end{align*}
Thus we can separate:
\[
Z_{u_j}=A_{u_j}+B_{u_j}
\]
where 
\begin{align*}
A_{u_j}&:=\sum_{\ell=1}^k \pi_\ell \left[  \frac1n\sum_{i=1}^n Z_i^\top R^{\ell-1}u_j 
-
\mathbb{E}  Z^\top R^{\ell-1} u_j \right],\\
B_{u_j}&:=
\sup_{\|\theta\|\le r} \left\{
 \frac1n\sum_{i=1}^n\left[  \sum_{\ell=1}^k (w_\ell(Z_i;\theta)-\pi_\ell)Z_i^\top R^{\ell-1} u_j\right]\right.\\
&\left.\qquad\qquad\;-\;
\mathbb{E}\left[\sum_{\ell=1}^k (w_\ell(Z;\theta)-\pi_\ell)Z^\top R^{\ell-1} u_j\right] 
\right\}.
\end{align*}
Noting that $Z_i^\top R^{\ell-1} u_j\,\,\stackrel{\text{iid}}{\sim}\,\,\mathcal{N}(0,1)$ and that $N\le17^d$, standard concentration bounds yield that
$$
\Pr\left[\max_{j\in[N]}A_{u_j}\le\sqrt{\frac{d\log17+\log(1/\delta)}{n}}\right]\ge1-\delta.
$$

We proceed to bounding $B_{u_j}$ uniformly over $\|\theta\|\le r$. A standard symmetrization trick (using Rademacher signs $\{\varepsilon_i\}$) shows
\[
\mathbb{E}\bigl[\exp\bigl(\lambda\,B_{u_j}\bigr)\bigr]
\;\;\le\;\;
\mathbb{E}\exp\left(
  \sup_{\|\theta\|\le r}
  \tfrac{2\lambda}{n}\sum_{i=1}^n \varepsilon_i\sum_{\ell=1}^k(w_\ell(Z_i;\theta)-\pi_\ell)(Z_i^\top R^{\ell-1}u_j)
\right).
\]
Next, we note that the map $\theta\mapsto \sum_{\ell=1}^k (w_\ell(x;\theta)-\pi_\ell)(x^\top R^{\ell-1} u)$ is Lipschitz in the quantity $\theta^\top x x^\top u$.  By the Ledoux--Talagrand contraction principle, the supremum over $\theta$ can be replaced up to constant factors with a supremum of $\|\theta\|\left\|\tfrac1n\sum_{i=1}^n \varepsilon_i\,Z_iZ_i^\top\right\|_{\mathrm{op}}$.  Since we restrict $\|\theta\|\le r$, this is at most
\[
r \;\Bigl\|\tfrac1n \sum_{i=1}^n \varepsilon_i\,Z_iZ_i^\top\Bigr\|_{\mathrm{op}},
\]
and the operator norm $\|\cdot\|_{\mathrm{op}}$ can be discretized again via $\max_{j} \sum_{i=1}^n \varepsilon_i\,(Z_i^\top u_j)^2$ with $u_j\in \mathcal{S}^d$.  

Because $Z_i^\top u_j\sim\mathcal{N}(0,1)$ (for fixed $u_j$), the square $(Z_i^\top u_j)^2$ is sub-exponential. Standard concentration bounds (cf.\ \cite{vershynin2018high}) then show
\[
\Pr\left[
  B_{u_j}
  \;\le\;
  C\,r\,\sqrt{\frac{d + \log(1/\delta)}{n}}
\right]
\;\ge\;1-\delta,
\]
for some universal constant $C>0$.  A union bound over $j=1,\dots,N$ then introduces a factor at most $N \le 17^d$, which can be absorbed into the $\log(1/\delta)$ term.

Combining the above steps and recalling $Z \le \tfrac{8}{7}\max_j Z_{u_j}$ completes the proof, establishing the stated high-probability deviation bound.

\end{proof}

\end{document}